\documentclass[10pt, conference, compsocconf]{IEEEtran}
\ifCLASSINFOpdf
\else
\fi
\usepackage{url}
\usepackage{cite}
\usepackage{graphicx}
\usepackage[cmex10]{amsmath}
\usepackage{amssymb, amsthm}
\usepackage{bm}
\usepackage{booktabs}
\usepackage{algorithm}
\usepackage{algorithmic}
\usepackage{multirow}

\begin{document}
%
\title{Multiple Learning for Regression in Big Data}

\author{\IEEEauthorblockN{Xiang Liu\IEEEauthorrefmark{1}}
\IEEEauthorblockA{Purdue University\\
xiang35@purdue.edu}
\and
\IEEEauthorblockN{Ziyang Tang\IEEEauthorrefmark{1}}
\IEEEauthorblockA{Purdue University\\
tang385@purdue.edu}
\and
\IEEEauthorblockN{Huyunting Huang}
\IEEEauthorblockA{Purdue University\\
huan1182@purdue.edu}
\and
\IEEEauthorblockN{Tonglin Zhang}
\IEEEauthorblockA{Purdue University\\
tlzhang@purdue.edu}
\and
\IEEEauthorblockN{Baijian Yang}
\IEEEauthorblockA{Purdue University\\
byang@purdue.edu}
\thanks{* indicates equal contribution}
}


%
\maketitle

\begin{abstract}
Regression problems that have closed-form solutions are well understood and can be easily implemented when the dataset is small enough to be all loaded into the RAM. Challenges arise when data are too big to be stored in RAM to compute the closed form solutions.
Many techniques were proposed to overcome or alleviate the memory barrier problem but the solutions are often local optima. In addition, most approaches require loading the raw data to the memory again when updating the models. Parallel computing clusters are often expected in practice if multiple models need to be computed and compared. 
We propose multiple learning approaches that utilize an array of sufficient statistics (SS) to address the aforementioned big data challenges. The memory oblivious approaches break the memory barrier when computing regressions with closed-form solutions, including but not limited to linear regression, weighted linear regression, linear regression with Box-Cox transformation (Box-Cox regression) and ridge regression models. 
The computation and update of the SS arrays can be handled at per row level or per mini-batch level. And updating a model is as easy as matrix addition and subtraction. Furthermore, the proposed approaches also enable the computational parallelizability of multiple models because multiple SS arrays for different models can be computed simultaneously with a single pass of slow disk I/O access to the  dataset.
We implemented our approaches on Spark and evaluated over the simulated datasets. Results showed our approaches can achieve exact solutions of multiple models. The training time saved compared to the traditional methods is proportional to the number of models need to be investigated.
\end{abstract}

\begin{IEEEkeywords}
Big Data; Linear Regression; Weighted Linear Regression; Ridge Regression; Box-Cox Transformation
\end{IEEEkeywords}
%
\IEEEpeerreviewmaketitle
\renewcommand{\figurename}{Fig.}
\newcommand\figref{Fig. \ref}
\newcommand\tbref{Table \ref}
\newcommand\algref{\textbf{Algorithm} \ref}

\newtheorem{definition}{Definition}
\newcommand\defref{\textbf{Definitoin} \ref}
\newtheorem{theorem}{Theorem}
\newtheorem{corollary}{Corollary}[theorem]
\newtheorem{lemma}[theorem]{Lemma}

\newcommand\secref{Section \ref}
\newcommand{\norm}[1]{\left\lVert#1\right\rVert}

\newcommand{\pluseq}{\mathrel{+}=}
\section{Introduction}\label{sec:intro}
Linear regression, weighted linear regression, linear regression with Box-Cox transformation (Box-Cox regression) and ridge regression have powered the society in many respects by modeling the relationship between a scalar response variable and explanatory variable(s). From housing price prediction to stock price prediction, and from face recognition to marketing analysis, the related applications span a wide spectrum~\cite{naseem2010linear,altay2005stock,nghiep2001predicting}.
After entering the big data era, these regression models are still prevalent in academia and industry. Even though more advanced models, such as XGBoost and deep learning, have seen significant successes lately, the regression models continue their impact in many fields due to their transparency, reliability and explainability~\cite{chen2016xgboost,he2016deep}. 
However, it is not easy to compute these models if the dataset is massive. Closed-form solutions would be impossible if the physical memory cannot hold all the data or the intermediate results needed for the computation. And trade-offs must be made between the accuracy and the time if the iterative methods should to be applied.
Hence, it is of high value to propose a set of big-data oriented approaches that can preserve the benefits of linear, weighted linear, Box-Cox and ridge regression.

For linear regression, academia and industry resort to two major techniques, ordinary least squares (OLS) and the iterative methods. The OLS method is designed to calculate the closed-form solution~\cite{kenney1962linear}. By solving the normal equation, OLS can immediately derive the solution from the data. The normal equation consists of $(\mathbf{X}^\top \mathbf{X})^{-1}$, if $\mathbf{X}^\top \mathbf{X}$ is singular, the normal equation will become unsolvable. One solution is to use generalized inverse~\cite{dresden1920,barata2012moore,ben2003generalized}. Although OLS is efficient time-wise in deriving the closed-form solution, it also introduces the memory barrier issue in that the RAM needs to be big enough to store the entire dataset to solve the equation. To overcome the memory barrier, the distributed matrix could be applied to perform the calculation as a remedy~\cite{moler1986matrix}. But the time cost makes this algorithm infeasible nevertheless. Due to this reason, the applications of this technique are limited. And another technique, the iterative methods, which include gradient descent, Newton's method and Quasi-Newton's method, are commonly used to provide approximate solutions. \cite{kiwiel2001convergence,wedderburn1974quasi}. 
 
Gradient descent, also known as steepest descent, targets to find the minimum of a function. It approaches the minimum by taking steps along the negative gradient of the function with a learning rate proportional to the gradient. It is more universal than OLS as the variations, such as mini-batch gradient descent and stochastic gradient descent, overcome the memory barrier issue by performing a calculation in small batches instead of feeding all the data into memory at once~\cite{ruder2016overview}. But, gradient descent oscillates around the minimum region when the algorithm gets close to the minimum. And its asymptotic rate of convergence is inferior to many other iterative methods. If an easier approach to the minimum or higher asymptotic rate of convergence is demanded, Newton's method is an alternative.
 
Newton's method is a root-finding algorithm, utilizing the Taylor series. To find a minimum/maximum, it needs the knowledge of the second derivative. Unlike gradient descent, this strategy enables Newton's method to approach the extrema/optima more easily rather than oscillations. Besides, it has been proven that Newton's method has the quadratic asymptotic rate of convergence. However, this algorithm is faster than gradient descent only if the Hessian matrix is known or easy to compute~\cite{wedderburn1974quasi}. Unfortunately, the expressions of the second derivatives for large scale optimization problem are often complicated and intractable.
 
Quasi-newton methods, for instance, DFP, BFGS and L-BFGS, were proposed as alternatives to Newton's method when the Hessian matrix is unavailable or too expensive to calculate~\cite{davidon1991variable,avriel2003nonlinear,malouf2002comparison}. Instead of inverting the Hessian matrix in Newton's method, quasi-newton methods build up an approximation for the inverse matrix to reduce the computational load. With this mechanism, quasi-newton methods are usually faster than Newton's method for large datasets. In linear regression, L-BFGS, a variation of BFGS, is one of the most widely used quasi-newton method~\cite{zaharia2010spark}. Generally, L-BFGS outperforms gradient descent in linear regression. 

For the aforementioned approaches, the majority of them require multiple pass through the dataset. Donald Knuth proposed an efficient solution which requires only single-pass through the dataset, however, this approach is only applicable for variance computation~\cite{knuth2014art}.
 
Weighted linear regression is a more generalized version of linear regression by quantifying the importance of different observations~\cite{myers1990classical}. A weighted version of OLS is designed to obtain the corresponding closed-form solution. The iterative methods with slight modifications are also applicable to weighted linear regression~\cite{holland1977robust}.

For Box-Cox regression, it is linear regression with the response variable changed by Box-Cox transformation~\cite{box1964analysis, sakia1992box}. The design philosophy of Box-Cox regression is to handle non-linearity between the response variable and explanatory variables by casting power transformation on the response variable. Naturally, approaches for linear regression are applicable to Box-Cox regression.
 
As linear regression is deficient in handling highly-correlated data, ridge regression is then proposed~\cite{hoerl1970ridge}. The basic idea of ridge regression is to add a $\ell_2$ penalty term to the error sum of squares (SSE) cost function of linear regression~\cite{hoerl1970ridge,marquaridt1970generalized}. A constrained version of OLS can solve this problem, producing similar closed-form solution. The only difference is that the $(\mathbf{X}^\top \mathbf{X})^{-1}$ component from OLS is substituted by $(\mathbf{X}^\top \mathbf{X} + \lambda \mathbf{I})^{-1}$, where $\lambda$ is the coefficient of $\ell_2$ penalty, and $\mathbf{I}$ is the identity matrix. By means of $\lambda \mathbf{I}$, the constrained OLS no longer has to deal with the singularity issue but the memory barrier issue from OLS remains. Gradient descent, Newton's method and quasi-newton methods as well can be applied~\cite{kiwiel2001convergence,wedderburn1974quasi,dennis1977quasi}. 

From the above discussions, it can be concluded that research gaps remain in the following two perspectives:
(i) OLS and its extended versions are difficult in handling the memory barrier issue; and
(ii) The iterative methods are time inefficient and require many iterations to well-train regression models. 
In addition, parameter tuning is inevitable under most conditions. It may probably take several days or even weeks for large scale projects to accomplish the desired performance goals of models. For Box-Cox regression or ridge regression, the situation gets worse as a set of power or ridge parameters are usually applied to pick the best one, which, of course, also multiply the time cost~\cite{pedregosa2011scikit}.

In order to integrate the pros of OLS based approaches that use closed-form solutions to produce the exact results and the iterative methods that overcome the memory barrier, we propose multiple learning approaches that utilize sufficient statistics (SS). The main contributions of our algorithms are summarized as below:
\begin{itemize}
    \item We introduced a SS array which can be computed at per row or per mini-batch level for calculating closed-form solutions. 
    \item Once the closed-form solutions are obtained, the optimums are found, i.e., the prediction performance is at least as good as OLS.
    \item With SS, the datasets stored in the large secondary storage, such as HDD or SSD, needs to be loaded to the primary storage one time only. The time efficiency is therefore greatly improved in contrast to the iterative methods that require multiple slow disk I/Os.
    \item Because multiple SS arrays for different models can be computed simultaneously, multiple models can be computed and updated with a single pass of the entire dataset with one iteration of slow disk I/Os.
\end{itemize}

\section{Background Concepts} \label{sec:bg}

For regression analysis, not only the estimators of the regression model coefficients $\hat{\bm{\beta}}$ are required, but also the estimators of variance $\sigma^2$ and the variance-covariance matrices $\hat{\rm{V}}(\hat{\bm{\beta}})$ should be computed for significant test.
For the ease of presentation, necessary notions and notations closely relevant to linear regression, weighted linear regression, Box-Cox regression and ridge regression are explained below.

\subsection{Linear Regression} 
Assume the dataset contains $n$ observations each of which has $p-1$ features. Consider a linear regression model
\begin{align} \label{eq:linreg}
    \bm{y}=\mathbf{X}\bm\beta + \bm{\varepsilon}
\end{align}
where $ \bm{y}=(y_1,y_2,\dots,y_n)^\top $ is a $n \times 1$ vector of the response variables, $ \mathbf{X}=(\bm{x}_1^\top, \bm{x}_2^\top,\dots,\bm{x}_n^\top)^\top $ is a $n \times p$ matrix of explanatory variables, $\bm{\beta} = (\beta_0,\beta_1,\dots, \beta_{p-1})^\top$ is a $p \times 1$ vector of regression coefficient parameters, and $\bm{\varepsilon} = (\varepsilon_1,\dots,\varepsilon_n)^\top $ is the error term which is a $n \times 1$ vector following the normal distribution $\mathcal{N}(0,\sigma^2\mathbf{I})$.

Linear regression is usually solved by maximizing loglikelihood function \eqref{eq:linreg_llh}.
\begin{align} \label{eq:linreg_llh}
\begin{split}
	\mathcal{L}_{lr}(\bm{\beta}, \sigma^2) =  -\frac{n}{2}\log(2\pi\sigma^2)
    -\frac{1}{2\sigma^2} \norm{\mathbf{y}-\mathbf{X}\bm\beta}_2^2
\end{split}
\end{align}
where $\norm{\cdot}_2$ is an $\ell_2$ norm.


The estimators of model coefficients, variance and variance-covariance matrix are shown in \eqref{eq:linreg_estimators}.
\begin{align}
\begin{split} \label{eq:linreg_estimators}
    & \hat{\bm{\beta}} = (\mathbf{X}^\top \mathbf{X})^{-1} \mathbf{X}^\top \mathbf{y} \\
    & \hat{\sigma}^2 =\frac{1}{n}(\mathbf{y}-\mathbf{X}\hat{\bm{\beta}})^\top (\mathbf{y}-\mathbf{X}\hat{\bm{\beta}}) \\
    & \hat{\rm{V}}(\hat{\bm{\beta}}) =
	\hat{\sigma}^2 (\mathbf{X}^\top\mathbf{X})^{-1}
\end{split}
\end{align}

Note that $\mathbf{x}_1^\top, \mathbf{x}_2^\top,\dots,\mathbf{x}_n^\top$ are all known observations. This means, the value of $\mathbf{x}_i^j$ can be easily computed and included as an explanatory variable in equation \eqref{eq:linreg}. As a result, this approach can also be used to fit polynomial regressions models, in addition to linear regression models. 

\subsection{Weighted Linear Regression} 
The weighted linear regression is similar to linear regression, except it assumes all the off-diagonal entries of the correlation matrix of the residuals are $0$. By means of minimizing the corresponding SSE cost function in \eqref{eq:wlreg_sse}, the estimators of the model coefficients, variance and variance-covariance matrix are shown in \eqref{eq:wlreg_estimators}.
\begin{align} \label{eq:wlreg_sse}
    SSE_{wlr}(\bm{\beta}_w) = \norm{\mathbf{W}^{1/2}(\mathbf{y}-\mathbf{X}\bm{\beta}_w)}_2^2 
\end{align}
where $\mathbf{W}$ is a diagonal matrix of weights.
\begin{align}
\begin{split} \label{eq:wlreg_estimators}
    & \hat{\bm{\beta}}_w = (\mathbf{X}^\top \mathbf{WX})^{-1} \mathbf{X}^\top \mathbf{Wy} \\
    & \hat{\sigma}_w^2 =\frac{1}{n} (\mathbf{y}-\mathbf{X}\hat{\bm{\beta}}_w)^\top \mathbf{W} (\mathbf{y}-\mathbf{X}\hat{\bm{\beta}}_w) \\
    & \hat{\rm{V}}(\hat{\bm{\beta}}_w) =\hat{\sigma}_w^2 (\mathbf{X}^\top\mathbf{WX})^{-1}
\end{split}
\end{align}

\subsection{Box-Cox Regression} 
Box-Cox regression model is a linear regression model with an additional power transformation on the response variable, as shown in~\eqref{eq:bcreg}.
\begin{align} \label{eq:bcreg}
    \mathbf{y}^{(c)}=\mathbf{X}\bm{\beta}_c + \bm{\varepsilon}
\end{align}
where $\mathbf{y}^{(c)}$ is the element-wise power transformation defined in \eqref{eq:bcreg_pp}.
\begin{align} \label{eq:bcreg_pp}
    \mathbf{y}^{(c)}=
    \begin{cases}
        (\mathbf{y}^c-1)/c \text{ if $c \neq 0$ } \\
        \log \mathbf{y} \text{ if $c=0$ }
    \end{cases}
\end{align}

Normally, a set $C$ of power parameters are applied to the response variable. In this case, for every $c \in C$, the one maximizes the profile loglikelihood \eqref{eq:bcreg_llh} is chosen as the best power parameter.

\begin{align} \label{eq:bcreg_llh}
\begin{split}
    & \mathcal{L}_{bc} (c,\bm{\beta}_c,\sigma_c^2) = 
    -\frac{n}{2}\log(2\pi)-\frac{n}{2}\log\sigma_c^2 \\
    & -\frac{1}{2\sigma_c^2} (\mathbf{y}^{(c)}-\mathbf{X}\bm{\beta}_c)^\top (\mathbf{y}^{(c)}-\mathbf{X}\bm{\beta}_c) - (\mathbf{c-1})^\top \log \mathbf{y}
\end{split}
\end{align}

The estimator of the model coefficients, variance and variance-covariance matrix for Box-Cox regression are
\begin{align} \label{eq:bcreg_estimators}
\begin{split}
    & \hat{\bm{\beta}}_c = (\mathbf{X}^\top \mathbf{X})^{-1} \mathbf{X}^\top \mathbf{y}^{(c)} \\
    & \hat{\sigma}_c^2 = \frac{1}{n} 
    (\mathbf{y}^{(c)}-\mathbf{X}\hat{\bm{\beta}}_c)^\top (\mathbf{y}^{(c)}-\mathbf{X}\hat{\bm{\beta}}_c) \\
    & \hat{\rm{V}}(\hat{\bm{\beta}_c}) =
	\hat{\sigma}_c^2 (\mathbf{X}^\top\mathbf{X})^{-1}
\end{split}
\end{align}

\subsection{Ridge Regression}
Ridge regression is linear regression with an $\ell_2$ penalty term added. The corresponding SSE cost function is:
\begin{align} \label{eq:ridge_sse}
    SSE_{ridge}(\lambda,\bm{\beta}_\lambda) =
    \norm{\mathbf{y}-\mathbf{X}\bm{\beta}_\lambda}_2^2
    + n\lambda \norm{\bm{\beta}_\lambda}_2^2
\end{align}
where  $\lambda$ is a non-negative tuning parameter used to control the penalty magnitude. For any $\lambda \geq 0$, \eqref{eq:ridge_sse} can be analytically minimized, yielding the estimator of $\bm{\beta}$ as
\begin{align} 
    & \hat{\bm{\beta}}_\lambda=(\mathbf{X}^\top \mathbf{X}+\lambda\mathbf{I})^{-1}\mathbf{X}^\top \mathbf{y} \nonumber\\
    & \hat{\sigma}_\lambda^2 =\frac{1}{n}
    (\mathbf{y}-\mathbf{X}\hat{\bm{\beta}}_\lambda)^\top (\mathbf{y}-\mathbf{X}\hat{\bm{\beta}}_\lambda) \label{eq:ridge_estimators}\\
    & \hat{\rm{V}}(\hat{\bm{\beta}}_\lambda) =
	\hat{\sigma}_\lambda^2 (\mathbf{X}^\top \mathbf{X}+\lambda \mathbf{I})^{-1}\mathbf{X}^\top\mathbf{X}(\mathbf{X}^\top \mathbf{X}+\lambda \mathbf{I})^{-1} \nonumber
\end{align}

\section{Methodology} \label{sec:method}
The main goal is to find approaches that are able to overcome the memory barrier issue of closed-form solutions and make them as widely applicable as the iterative methods in big data. 
In pursuit of this goal, the array of sufficient statistics (SS) is formally defined. And SS based multiple learning algorithms are proposed in this section.


\subsection{Sufficient Statistics Array}
SS array is an array of sufficient statistics used to calculate the estimators of the models and the loglikelihood function (or SSE cost function) without a second visit to the dataset. It's inspired by the computation-wise row-independent of
 the equivalent forms of \eqref{eq:linreg_estimators} of linear regression \cite{zhang2017box,zhang2017exact}.
 
Rewritting $\hat{\bm{\beta}}$ from \eqref{eq:linreg_estimators} in \eqref{eq:linreg_estimators_sum}, $\sum_{i=1}^n\mathbf{x}_i^\top\mathbf{x}_i$ is computation-wise row independent, i.e., for any two observations $\mathbf{x}_{i1}$ and $\mathbf{x}_{i2}$ , calculating the summation of $\mathbf{x}_{i1}^\top\mathbf{x}_{i1}$ doesn't depend on $\mathbf{x}_{i2}$. Likewise, $\sum_{i=1}^n \mathbf{x}_iy_i$ and $\sum_{i=1}^{n}{y_i^2}$ are computation-wise row-independent as well.
\begin{flalign}
\begin{split} \label{eq:linreg_estimators_sum}
    & \hat{\bm{\beta}} =
    \left( \sum_{i=1}^{n}\mathbf{x}_i^\top\mathbf{x}_i \right)^{-1}
    \left( \sum_{i=1}^{n}\mathbf{x}_iy_i \right) \\
\end{split}
\end{flalign}

Inspired by this thought, the array of SS is formally defined as follows.
\begin{definition}
Sufficient statistics (SS) array is an array of sufficient statistics that computed at per row level or per mini batch level from the dataset and can be used to compute the estimators of the model coefficients $\hat{\bm{\beta}}$, the variance $\sigma^2$, the variance-covariance matrices $\hat{\rm{V}}(\hat{\bm{\beta}})$ and the loglikelihood (or SSE cost function) without revisiting the dataset.
\end{definition}

\subsection{Linear Regression}
Based on \eqref{eq:linreg_llh} and \eqref{eq:linreg_estimators}, $\mathcal{S}_{lr}$ is presented as an array of SS for linear regression. 
\begin{align} \label{eq:linreg_SS}
    \mathcal{S}_{lr} = (s_{yy}, \mathbf{s}_{xy}, \mathbf{S}_{xx}) 
    = (\sum_{i=1}^n s_{yy,i}, \sum_{i=1}^n \mathbf{s}_{xy,i}, \mathbf{S}_{xx,i})
\end{align}
where $s_{yy,i}=y_i^2$ is a scalar, $\mathbf{s}_{xy,i}=\mathbf{x}_iy_i$ is a $p \times 1$ vector, and $\mathbf{S}_{xx,i}=\mathbf{x}_i\mathbf{x}_i^\top$ is a $p \times p$ matrix.

By \eqref{eq:linreg_SS}, we obtain the following
\begin{align}
\begin{split} \label{eq:linreg_estimators_SS}
& \bm{\hat{\beta}}=\mathbf{S}_{xx}^{-1}\mathbf{s}_{xy} \\
    & \hat{\sigma}^2 = \frac{1}{n}    (s_{yy}-\mathbf{s}_{xy}^\top\mathbf{S}_{xx}^{-1}\mathbf{s}_{xy}) \\
    & \hat{\rm{V}}(\hat{\bm{\beta}})=\hat{\sigma}^2\mathbf{S}_{xx}^{-1} 
\end{split}
\end{align}

\begin{theorem}
    $\mathcal{S}_{lr}$ is an array of SS for linear regression to derive $\hat{\bm{\beta}}$, $\hat{\sigma}^2$, $\hat{\rm{V}}(\hat{\bm{\beta}})$ and $\mathcal{L}_{lr}(\bm{\beta}, \sigma^2)$. 
\end{theorem}

\begin{proof}
From \eqref{eq:linreg_SS}, the loglikelihood can be expressed as a functin of $\mathcal{S}_{lr}$.
\begin{align}
\begin{split} \label{eq:linreg_llh_SS}
    & \mathcal{L}_{lr}(\bm{\beta}, \sigma^2) = -\frac{n}{2}\log(2\pi\sigma^2) \\ & -\frac{1}{2\sigma^2}(s_{yy}-2\mathbf{s}_{xy}^\top\bm{\beta}+\bm{\beta}^\top \mathbf{S}_{xx} \bm{\beta})
\end{split}
\end{align}
which only depends on SS for linear regression.
\end{proof}


To accelerate the computation, row-by-row calculation could be optimized by batch-by-batch computation, i.e. $\sum_{i=1}^n y_i^2$, $\sum_{i=1}^n\mathbf{x}_iy_i$ and $\sum_{i=1}^n\mathbf{x}_i^\top\mathbf{x}_i$ could be written in the form of batch:
\begin{align}
\begin{split} \label{eq:linreg_estimators_SS_batch}
    & s_{yy} = \sum_{k=1}^m s_{yy}^{(k)} = \sum_{k=1}^m \mathbf{y}_k^\top \mathbf{y}_k \\
    & \mathbf{s}_{xy} = \sum_{k=1}^m \mathbf{s}_{xy}^{(k)} =\sum_{k=1}^m \mathbf{X}_k^\top \mathbf{y}_k \\
    & \mathbf{S}_{xx} = \sum_{k=1}^m \mathbf{S}_{xx}^{(k)} = \sum_{k=1}^m \mathbf{X}_k^\top \mathbf{X}_k 
\end{split}
\end{align}
where $m$ denotes the total number of batches, $s_{yy}^{(k)}$, $\mathbf{s}_{xy}^{(k)}$ and $\mathbf{S}_{xx}^{(k)}$ denotes SS array in batch $k$. $\mathbf{y}_k$ is a $m_k \times 1$ vector, $\mathbf{X}_k$ is a $m_k \times m_k$ array and $m_k$ is the batch size for batch $k$. The multiple learning approach for linear regression algorithm by mini-batch is shown in \algref{alg:linreg_SS_batch}.
\begin{algorithm}[tb]
\caption{Linear Regression with Sufficient Statistics}
\label{alg:linreg_SS_batch}
\begin{flushleft}
\textbf{Input}: batch-by-batch of the entire dataset\\
\textbf{Output}: $\hat{\bm{\beta}}$, $\hat{\sigma}^2$ and $\hat{\rm{V}}(\hat{\bm{\beta}})$
\end{flushleft}
\begin{algorithmic}[1] 
    \STATE $s_{yy}=0, \mathbf{s}_{xy}=\mathbf{0}, \mathbf{S}_{xx}=\mathbf{0}$
    \FOR{$k \gets 1$ to $m$}  
        \STATE Compute $s_{yy}^{(k)}, \mathbf{s}_{xy}^{(k)}, \mathbf{S}_{xx}^{(k)}$ based on \eqref{eq:linreg_estimators_SS_batch}
        \STATE $s_{yy}\pluseq s_{yy}^{(k)},\mathbf{s}_{xy} \pluseq \mathbf{s}_{xy}^{(k)},\mathbf{S}_{xx} \pluseq \mathbf{S}_{xx}^{(k)}$
    \ENDFOR
    \IF {$\mathbf{S}_{xx}$ is singular}
        \STATE Compute $\mathbf{S}_{xx}^{-1}$ using generalized inverse 
    \ELSE
        \STATE  Compute $\mathbf{S}_{xx}^{-1}$
    \ENDIF
    \STATE Compute $\hat{\bm{\beta}}$, $\hat{\sigma}^2$ and $\hat{\rm{V}}(\hat{\bm{\beta}})$ based on \eqref{eq:linreg_estimators_SS}
    \STATE \textbf{return} $\hat{\bm{\beta}}$, $\hat{\sigma}^2$ and $\hat{\rm{V}}(\hat{\bm{\beta}})$
\end{algorithmic}
\end{algorithm}

\subsection{Weighted Linear Regression}
Weighted linear regression uses weights to adjust the importance of different observations. Therefore, the SS array $\mathbf{S}_{wls}$ for weighted linear regression is slightly different.
\begin{align} \label{eq:wlreg_SS}
\begin{split}
    \mathcal{S}_{wlr} & = (s_{wyy}, \mathbf{s}_{wxy}, \mathbf{S}_{wxx}) \\
    & = (\sum_{i=1}^n s_{wyy,i}, \sum_{i=1}^n \mathbf{s}_{wxy,i}, \mathbf{S}_{wxx,i})
\end{split}
\end{align}
where $s_{wyy,i}=w_iy_i^2$ is scalar, $\mathbf{s}_{wxy,i}=\mathbf{x}_iw_iy_i$ is a $p \times 1$ vector, and $\mathbf{S}_{wxx,i}=w_i\mathbf{x}_i\mathbf{x}_i^\top$ is a $p \times p$ matrix.

The estimators are re-expressed as follows:
\begin{align}
\begin{split} \label{eq:wlreg_estimators_SS}
    & \bm{\hat{\beta}}_w=\mathbf{S}_{wxx}^{-1}\mathbf{s}_{wxy} \\
    & \hat{\sigma}_w^2 = \frac{1}{n}    (s_{wyy}-\mathbf{s}_{wxy}^\top\mathbf{S}_{wxx}^{-1}\mathbf{s}_{wxy}) \\
    & \hat{\rm{V}}(\hat{\bm{\beta}}_w)=\hat{\sigma}_w^2\mathbf{S}_{wxx}^{-1}
\end{split}
\end{align}

\begin{theorem}
    $\mathcal{S}_{wlr}$ is an array of SS for weighted linear regression to derive the estimators of $\hat{\bm{\beta}}_w$, $\sigma_w^2$, $\hat{\rm{V}}(\hat{\bm{\beta}}_w)$ and $SSE_{wls}(\bm{\beta}_w)$. 
\end{theorem}

\begin{proof}
From \eqref{eq:wlreg_SS}, \eqref{eq:wlreg_sse} can be expressed as a function of the SS array
\begin{align} \label{eq:wlreg_sse_SS}
    SSE_{wls}(\bm{\beta}_w)=s_{wyy}-2\mathbf{s}_{wxy}^\top\bm{\beta}_w+\bm{\beta}_w^\top \mathbf{S}_{wxx} \bm{\beta}_w
\end{align}
which only depends on SS for weighted linear regression.
\end{proof}


Similar to multiple learning approach for linear regression algorithm, calculating SS batch by batch is also feasible.
\begin{align}
\begin{split} \label{eq:wlreg_estimators_SS_batch}
    & s_{wyy} = \sum_{k=1}^m s_{wyy}^{(k)} = \sum_{k=1}^m \mathbf{y}_k^\top \mathbf{W}_k \mathbf{y}_k \\
    & \mathbf{s}_{wxy} = \sum_{k=1}^m \mathbf{s}_{wxy}^{(k)} =\sum_{k=1}^m \mathbf{X}_k^\top \mathbf{W}_k \mathbf{y}_k \\
    & \mathbf{S}_{wxx} = \sum_{k=1}^m \mathbf{S}_{wxx}^{(k)} = \sum_{k=1}^m \mathbf{X}_k^\top \mathbf{W}_k \mathbf{X}_k 
\end{split}
\end{align}
where $\mathbf{W}_k$ is a $m_k \times m_k$ diagonal weight matrix in batch $k$.

The multiple learning approach for weighted linear regressoin is shown in \algref{alg:wlreg_SS_batch}.

\begin{algorithm}[tb]
\caption{Weighted Linear Regression with Sufficient Statistics}
\label{alg:wlreg_SS_batch}
\begin{flushleft}
\textbf{Input}: batch by batch of the entire dataset\\
\textbf{Output}: $\hat{\bm{\beta}}$, $\hat{\sigma}^2$ and $\hat{\rm{V}}(\hat{\bm{\beta}})$
\end{flushleft}
\begin{algorithmic}[1] 
    \STATE $s_{wyy}=0, \mathbf{s}_{wxy}=\mathbf{0}, \mathbf{S}_{wxx}=\mathbf{0}$
    \FOR{$k \gets 1$ to $m$}  
        \STATE Compute $s_{wyy}^{(k)}, \mathbf{s}_{wxy}^{(k)}, \mathbf{S}_{wxx}^{(k)}$ based on \eqref{eq:wlreg_estimators_SS_batch}
        \STATE $s_{wyy}\pluseq s_{wyy}^{(k)},\mathbf{s}_{wxy} \pluseq \mathbf{s}_{wxy}^{(k)},\mathbf{S}_{wxx} \pluseq \mathbf{S}_{wxx}^{(k)}$
    \ENDFOR
    \IF {$\mathbf{S}_{wxx}$ is singular}
        \STATE Compute $\mathbf{S}_{wxx}^{-1}$ using generalized inverse 
    \ELSE
        \STATE  Compute $\mathbf{S}_{wxx}^{-1}$
    \ENDIF
    \STATE Compute $\hat{\bm{\beta}}_w$, $\hat{\sigma}_w^2$ and $\hat{\rm{V}}(\hat{\bm{\beta}}_w)$ based on \eqref{eq:wlreg_estimators_SS}
    \STATE \textbf{return} $\hat{\bm{\beta}}_w$, $\hat{\sigma}_w^2$ and $\hat{\rm{V}}(\hat{\bm{\beta}}_w)$
\end{algorithmic}
\end{algorithm}

\subsection{Box-Cox Regression}
Box-Cox regression requires a power transformation on the response variable. Commonly, a set $C$ of power parameters are applied. And the $c$ maximizes the \eqref{eq:bcreg_llh} is picked as the best parameter. As the profile loglikelihood is required for parameter picking, $(\mathbf{c-1})^\top \log \mathbf{y}$ is necessarily needed. 

The arrays of SS for Box-Cox regression is shown in \eqref{eq:bcreg_SS}. For every $c \in C$,
\begin{align}
\begin{split} \label{eq:bcreg_SS}
    & {\mathcal{S}}_{c, bc} = (s_{c,yy}, s_{logy}, \mathbf{s}_{c,xy}, \mathbf{S}_{xx}) \\
    & = (\sum_{i=1}^n s_{c,yy,i}, \sum_{i=1}^n s_{logy,i}, \sum_{i=1}^n \mathbf{s}_{c,xy,i}, \mathbf{S}_{xx,i})
\end{split}
\end{align}
where $s_{c,yy,i}=\left(y_i^{(c)}\right)^2$ and $s_{logy,i}=\log y_i$ are scalars, $\mathbf{s}_{c,xy,i}=\mathbf{x}_iy_i^{(c)}$ is a $p \times 1$ vector and $\mathbf{S}_{xx,i}=\mathbf{x}_i\mathbf{x}_i^\top$ is a $p \times p$ matrix. Notably, $\mathbf{S}_{xx}$ is sharable to all models.

Thus, for every $c \in C$, 
\begin{align}
\begin{split} \label{eq:bcreg_estimators_SS}
    & \bm{\hat{\beta}}_c=\mathbf{S}_{xx}^{-1}\mathbf{s}_{c,xy} \\
    & \hat{\sigma}_c^2 = \frac{1}{n}    (s_{c,yy}-\mathbf{s}_{c,xy}^\top\mathbf{S}_{xx}^{-1}\mathbf{s}_{c,xy}) \\
    & \hat{\rm{V}}(\hat{\bm{\beta}}_c)=\hat{\sigma}_c^2\mathbf{S}_{xx}^{-1} \\
\end{split}
\end{align}

\begin{theorem}
    For any $c \in C$, the corresponding $\mathcal{S}_{c,bc}$ is an array of SS for Box-Cox regressoin, which can be used to compute $\hat{\bm{\beta}}_c$, $\hat{\sigma}_c^2$, $\hat{\rm{V}}(\hat{\bm{\beta}}_c)$ and $\mathcal{L}_{bc}(c,\bm{\beta}_c,\sigma_c^2)$. 
\end{theorem}

\begin{proof} By \eqref{eq:bcreg_estimators_SS}, \eqref{eq:bcreg_llh} becomes
\begin{align} \label{eq:bcreg_llh_SS}
\begin{split}
    & \mathcal{L}_{bc}(\bm{\beta}_c, \sigma^2_c) = -\frac{n}{2}\log(2\pi\sigma_c^2) \\ & -\frac{1}{2\sigma_c^2}(s_{c,yy}-2\mathbf{s}_{c,xy}^\top\bm{\beta}_c+\bm{\beta}_c^\top \mathbf{S}_{c,xx} \bm{\beta}_c)+(c-1)s_{logy}
\end{split}
\end{align}
which only depends on SS for Box-Cox linear regression.
\end{proof}

Batched version of SS for any $c\in C$ is shown in \eqref{eq:bcreg_estimators_SS_batch}. 

\begin{align}
\begin{split} \label{eq:bcreg_estimators_SS_batch}
    & s_{c,yy} = \sum_{k=1}^m s_{c,yy}^{(k)} = \sum_{k=1}^m (\mathbf{y}_k^{(c)})^\top \mathbf{y}_k \\
    & \mathbf{s}_{c,xy} = \sum_{k=1}^m \mathbf{s}_{c,xy}^{(k)} =\sum_{k=1}^m \mathbf{X}_k^\top \mathbf{y}_k^{(c)} \\
    & \mathbf{S}_{xx} = \sum_{k=1}^m \mathbf{S}_{xx}^{(k)} = \sum_{k=1}^m \mathbf{X}_k^\top \mathbf{X}_k 
\end{split}
\end{align}
where $\mathbf{y}^{(c)}_k$ is a $m_k \times 1$ vector in batch $k$.

The SS-based Box-Cox regression algorithm by mini-batch is presented in \algref{alg:bcreg_SS_batch}.

\begin{algorithm}[tb]
\caption{Box-Cox Regression with Sufficient Statistics}
\label{alg:bcreg_SS_batch}
\begin{flushleft}
\textbf{Input}: batch by batch of the entire dataset\\
\textbf{Output}: $\hat{\bm{\beta}}_{best}$, $\hat{\sigma}_{best}^2$ and $\hat{\rm{V}}(\hat{\bm{\beta}}_{best})$
\end{flushleft}
\begin{algorithmic}[1] 
    \STATE$\mathbf{S}_{xx}=\mathbf{0}$
    \FOR{$c \in C$}
        \STATE $s_{c,yy}=0, \mathbf{s}_{c,xy}=\mathbf{0}$
    \ENDFOR
    \FOR{$k \gets 1$ to $m$}
        \STATE Compute $\mathbf{S}_{xx}$ based on \eqref{eq:bcreg_estimators_SS_batch}
        \STATE $\mathbf{S}_{xx} \pluseq \mathbf{S}_{xx}^{(k)}$
        \FOR{$c \in C$}
             \STATE Compute $s_{c,yy}^{(k)}$ and $\mathbf{s}_{c,xy}^{(k)}$ based on \eqref{eq:bcreg_estimators_SS_batch}
             \STATE $s_{c,yy} \pluseq s_{c,yy}^{(k)}, \mathbf{s}_{c,xy} \pluseq \mathbf{s}_{c,xy}^{(k)}$
        \ENDFOR
    \ENDFOR
    \IF {$\mathbf{S}_{xx}$ is singular}
        \STATE Compute $\mathbf{S}_{xx}^{-1}$ using generalized inverse 
    \ELSE
        \STATE  Compute $\mathbf{S}_{xx}^{-1}$
    \ENDIF
    \FOR{$c \in C$}
        \STATE Compute $\hat{\bm{\beta}}_c$, $\hat{\sigma}_c^2$ and $\hat{\rm{V}}(\hat{\bm{\beta}}_c)$ based on \eqref{eq:bcreg_estimators_SS}
        \STATE Compute $\mathcal{L}_{bc}$ based on \eqref{eq:bcreg_llh_SS}
    \ENDFOR
    \STATE \textbf{return} $\hat{\bm{\beta}}_{best}$, $\hat{\sigma}_{best}^2$ and $\hat{\rm{V}}(\hat{\bm{\beta}}_{best})$ based on $\mathcal{L}_{bc}$
\end{algorithmic}
\end{algorithm}

\subsection{Ridge Regression}
Although ridge regression requires a set $D$ of ridge parameters, the SS array is re-usable to all ridge parameters and could be borrowed directly from linear regression. 

Let $\mathcal{S}_{ridge} = \mathcal{S}_{lr}$, for every $\lambda \in D$, the corresponding estimators $\hat{\bm{\beta}}_\lambda$, $\hat{\sigma}_\lambda^2$, $\hat{\rm{V}}(\hat{\bm{\beta}_\lambda})$ and the SSE cost function are:
\begin{align}
\begin{split} \label{eq:ridge_estimators_SS}
    & \hat{\bm{\beta}}_\lambda = (\mathbf{S}_{xx}+\lambda\mathbf{I})^{-1}\mathbf{s}_{xy} \\
    & \hat{\sigma}_\lambda^2 = \frac{1}{n}    (s_{yy}-\mathbf{s}_{xy}^\top(\mathbf{S}_{xx}+\lambda\mathbf{I})^{-1}\mathbf{s}_{xy}) \\
    & \hat{\rm{V}}(\hat{\bm{\beta}}_\lambda)=\hat{\sigma}_\lambda^2(\mathbf{S}_{xx}+\lambda\mathbf{I})^{-1}\mathbf{S}_{xx}(\mathbf{S}_{xx}+\lambda\mathbf{I})
\end{split}
\end{align}

\begin{align}
    SSE_{ridge}(\lambda,\bm{\beta}_\lambda)=\norm{\mathbf{y}-\mathbf{X}\bm{\beta}_\lambda}_2^2 +n\lambda\norm{\bm{\beta}_\lambda}_{2}^2 \label{eq:ridge_sse_SS}
\end{align}
The best $\lambda$ is selected by the ridge trace method.

\begin{theorem}
    $\mathcal{S}_{ridge}$ is the SS array for ridge regression.
\end{theorem}

\begin{proof}
From \eqref{eq:ridge_estimators_SS}, \eqref{eq:ridge_sse} could be expressed as
\begin{align}
\begin{split}
    SSE_{ridge}(\lambda, & \bm{\beta}_\lambda) = \\ 
    & s_{yy}-2\mathbf{s}_{xy}^\top\bm{\beta}_\lambda+\bm{\beta}_\lambda^\top \mathbf{S}_{xx} \bm{\beta}_\lambda
    +n\lambda\bm{\beta}_\lambda^\top\bm{\beta}_\lambda
\end{split}
\end{align}
which only depends on SS for ridge regression.
\end{proof}



The batched version for SS is also identical to that of linear regression. The corresponding algorithm is presented in \algref{alg:ridge_SS_batch}.

\begin{algorithm}[tb]
\caption{Ridge Regression with Sufficient Statistics}
\label{alg:ridge_SS_batch}
\begin{flushleft}
\textbf{Input}: batch-by-batch of the entire dataset\\
\textbf{Output}: $\hat{\bm{\beta}}_{best}$, $\hat{\sigma}_{best}^2$ and $\hat{\rm{V}}(\hat{\bm{\beta}}_{best})$
\end{flushleft}
\begin{algorithmic}[1] 
    \STATE $s_{yy}=0, \mathbf{s}_{xy}=\mathbf{0}, \mathbf{S}_{xx}=\mathbf{0}$
    \FOR{$k \gets 1$ to $m$}  
        \STATE Compute $s_{yy}^k, \mathbf{s}_{xy}^k, \mathbf{S}_{xx}^k$ based on \eqref{eq:linreg_estimators_SS_batch}
        \STATE $s_{yy}\pluseq s_{yy}^k,\mathbf{s}_{xy} \pluseq \mathbf{s}_{xy}^k,\mathbf{S}_{xx} \pluseq \mathbf{S}_{xx}^k$
    \ENDFOR
    \FOR{$\lambda \in D$}
        \STATE  Compute $(\mathbf{S}_{xx}+\lambda\mathbf{I})^{-1}$
        \STATE Compute $\hat{\bm{\beta}}_\lambda$, $\hat{\sigma}_\lambda^2$ and $\hat{\rm{V}}(\hat{\bm{\beta}}_\lambda)$ based on \eqref{eq:ridge_estimators_SS}
        \STATE Compute $SSE_{ridge}$ and ridge trace
    \ENDFOR
    \STATE \textbf{return} $\hat{\bm{\beta}}_{best}$, $\hat{\sigma}_{best}^2$ and $\hat{\rm{V}}(\hat{\bm{\beta}}_{best})$ by ridge trace
\end{algorithmic}
\end{algorithm}
\section{Experiments}
To evaluate the proposed multiple learning algorithms, extensive experiments were conducted on a four-node Spark cluster. All the algorithms were implemented and tested on Spark. 


\begin{table}[htbp]
\centering
\caption{Configurations of Clusters}
\label{tb:cluster_config}
\begin{tabular}{|c|c|c|c|c|}
\hline
 & Master & Slave1 & Slave2 & Slave3 \\ \hline
CPU & i7-3770 & i7-3770 & Quad Q8400 & Quad Q9400 \\ \hline
Memory & 16GB & 16GB & 4GB & 4GB \\ \hline
Disk & 1TB & 1TB & 250GB & 250GB\\ \hline
\end{tabular}
\end{table}

\subsection{Setup}
The 4-node Spark cluster was configured with 1 master node and 3 worker nodes. 
The hardware specs of each of the four computers are shown in
\tbref{tb:cluster_config}.

\subsubsection{Data Simulation} \text{ } \\
To understand how massive datasets could impact the computing, we simulated 3 datasets with 0.6 million, 6 million and 60 million observations. The sizes of these datasets are approximately 1GB, 10GB, and 100GB. Generally, the 1GB and 10GB datasets can be loaded into memory easily. However, the 100GB dataset cannot be entirely loaded into the memory at one time. Each row of the data has 100 features for the experiments and all the features are of double type and continuous variables. In each response $y$, the corresponding error follows the normal distribution, i.e. $\bm{\varepsilon} \sim \mathcal{N}(0,\mathbf{I})$. Additionally, another 3 similar datasets are generated with all the responses set to be positive for proper Box-Cox regression.


\subsubsection{Experiment Design} \text{ } \\
We designed two experiments, one for time performance and the other for prediction quality, to compare the results between the multiple learning algorithms and the traditional ones on Spark.

\begin{table}[tb]
\centering
\caption{Time Performance Comparison. Spark represents the traditional approaches implemented by Apache Spark; SS 1 (SS 128) means the multiple learning approaches with batch size fixed to $1$ ($128$); $W=\mathbf{I}$ denotes the weights of the observations; $C=[-1.5 \text{ to } 1.5]$ represents the power parameters for Box-Cox regression from $-1.5$ to $1.5$ by an interval of $0.1$. Likewise, $D=[0 \text{ to } 0.9]$ are the ridge parameters from $0$ to $0.9$ by an interval of 0.1.}
\begin{tabular}{|cc|c|c|c|}
\hline
\multicolumn{2}{|c|}{\textbf{Model}} & \multicolumn{3}{c|}{\textbf{Time Used} (s)} \\ \cline{3-5}
                        &            & 1GB          & 10GB          & 100GB     \\\hline
LR                      & Spark      & 41.86        & 338.27        & 3266.16   \\\cline{2-5}
                        & SS 1       & 19.59        & 154.16        & 1505.64   \\\cline{2-5}
                        & SS 128     & 15.67        & 126.33        & 1267.96   \\\hline
Weighted LR             & Spark      & 42.23        & 339.54        & 3263.37   \\\cline{2-5}
$W=\mathbf{I}$          & SS 1       & 19.76        & 155.47        & 1528.75   \\\cline{2-5}
                        & SS 128     & 16.73        & 125.35        & 1289.54   \\\hline
Box-Cox                 & Spark      & 42.63        & 341.31        & 3264.33   \\\cline{2-5}
$C=[1]$                 & SS 1       & 19.16        & 156.41        & 1532.00   \\\cline{2-5}
                        & SS 128     & 15.19        & 122.49        & 1200.49   \\\hline
Box-Cox                 & Spark      & 431.29       & 3429.34       & \textbf{33701.51}  \\\cline{2-5}
$C=[-1.5 \text{ to } 1.5]$ & SS 1       & 19.87        & 160.13        & \textbf{1674.62}   \\\cline{2-5}
                        & SS 128     & 16.52        & 122.21        & \textbf{1206.17}   \\\hline
Ridge                   & Spark      & 41.58        & 328.48        & 3276.10   \\\cline{2-5}
$D=[0.1]$               & SS 1       & 19.87        & 152.47        & 1620.46   \\\cline{2-5}
                        & SS 128     & 16.10        & 127.92        & 1213.64   \\\hline
Ridge                   & Spark      & 423.63       & 3342.58       & \textbf{32688.28}  \\\cline{2-5}
$D=[0 \text{ to } 1.9]$ & SS 1       & 20.56        & 154.34        & \textbf{1651.33}   \\\cline{2-5}
                        & SS 128     & 16.80        & 125.63        & \textbf{1230.45}   \\\hline
                        
\end{tabular}
\label{tb:time_performance}
\end{table}

\begin{table}[tb]
\centering
\caption{MSE comparison}
\begin{tabular}{|cc|c|c|c|}
\hline
\multicolumn{2}{|c|}{\textbf{Model}} & \multicolumn{3}{c|}{\textbf{MSE} (s)} \\ \cline{3-5}
                        &            & 1GB          & 10GB          & 100GB     \\\hline
LR                      & Spark      & 1009520.77   & 993455.96     & 994025.56   \\\cline{2-5}
                        & SS         & 1009520.77   & 993455.96     & 994025.56   \\\hline
Weighted LR             & Spark      & 1009520.77   & 993455.96     & 994025.56   \\\cline{2-5}
$W=\mathbf{I}$          & SS         & 1009520.77   & 993455.96     & 994025.56   \\\hline
Box-Cox                 & Spark      & 1138432.54   & 1053491.23    & 1011557.43 \\\cline{2-5}
$C=[1]$                 & SS         & 1138432.54   & 1053491.23    & 1011557.43   \\\hline
Ridge                   & Spark      & 1009520.77   & 993455.96     & 994025.56   \\\cline{2-5}
$D=[0.1]$               & SS         & 1009520.77   & 993455.96     & 994025.56   \\\hline

\end{tabular}
\label{tb:mse}
\end{table}

\textbf{Experiment I: Time Performance Comparison}

The first experiment is to evaluate the time used for training different models. In this experiment, we compared the time performance of the multiple learning approaches with the traditional approaches. For the multiple learning approaches, we measured the time performance with regard to different batch sizes.

\textbf{Experiment II: Prediction Quality Comparison}

To experimentally support that our algorithms are as accurate as OLS algorithms with one pass through the datasets, we compared our algorithms with the traditional ones. In this experiment, we used 1GB, 10GB, and 100GB as the training sets and an additional 0.2GB, 2GB and 20GB data for testing (the testing sets are sampled in accordance with the same strategy for the generation of the training sets). To compare the prediction quality, Mean Squared Error (MSE), defined in equation \eqref{eq:mse}, is used as performance matircs. 
\begin{align} \label{eq:mse}
    MSE = \frac{\sum_i^n (y_i - \hat{y}_i)^2}{n}
\end{align}
where $y_i$ is the real value for observation $i$ and $\hat{y}_i$ is the predicted value, $n$ is the total number of observations.

\subsection{Results}
\tbref{tb:time_performance} and \tbref{tb:mse} show the results of two experiments.

\textbf{Experiment I: Time Performance Comparison}


Based on the results from \tbref{tb:time_performance}, the training time of our methods is twice efficient than that of the traditional ones on Spark. However, it's mainly ascribed to the embedded model summary functionality of Spark which requires a second visit to the dataset. Excluding this factor, the performance of our algorithms are nearly the same as the traditional ones on Spark.
But for model training with multiple parameters (e.g. model selection) from a set of candidate models, the proposed multiple learning has a great advantage. As is shown in \tbref{tb:time_performance}, the computation time needed to perform  traditional Box-Cox and ridge regression are affected drastically by the number of power parameters and ridge parameters. In contrast, the time overhead of the proposed multiple learning algorithms increased marginally by computing multiple parameters (or multiple models) simultaneously with multiple SS arrays. In \tbref{tb:time_performance}, our approaches are almost 20 times faster than the traditional approaches on Spark when computing 31 Box-Cox models or 20 Ridge regression models for the batch size $=1$. Speed-up factors can be further increased to around 27 if we increased the batch size to 128, \textit{i.e.} the sufficient statistical arrays are updated every 128 rows. 
Essentially, the training time saved with the multiple learning approach is proportional to the number of models needed to train.

It is also evident in \tbref{tb:time_performance} that bigger batch size also decreases the training time. The effect of batch size becomes more significant when the data size is larger. Comparing batch size of 128 against batch size of 1, the time reduction for data size of 1GB, 10GB and 100GB dataset are approximately 16\%, 22\%, and 30\%, respectively.  It can be inferred that more time is likely to be saved with bigger batch size for larger datasets.

From experiment I, we conclude that if model selection is needed for a given large scale dataset, the proposed multiple learning approach can significantly outperform the traditional approaches by reducing the disk I/Os to one time. This feature is highly desirable when multiple models need to be calculated and compared in real life applications.

\textbf{Experiment II: Prediction Quality Comparison}

\tbref{tb:mse} shows the prediction quality, using MSE, for the multiple learning approaches and the traditional ones given 1GB, 10GB, and 100GB datasets. As expected, the prediction accuracy of our approaches is identical to the built-in spark algorithms, providing experimental support to the proof presented in Section 3. Given the same accuracy, the proposed approaches outperformed the traditional approaches with with faster training time. And the larger the datasets, the more advantageous the proposed methods are.

\section{Conclusion}
In this paper, the multiple learning approaches for regression are proposed for big data. With only one pass through the dataset, a SS array is computed to derive the closed-form solutions for linear regression, weighted linear regression, Box-Cox regression and ridge regression. Theoretically and experimentally, it's proven that multiple learning is capable of overcoming the memory barrier issue. 

Furthermore, multiple SS arrays could be applied to obtain multiple models at once. Unlike other traditional methods that can only learn one model at a time, multiple learning outperforms the traditional techniques as far as time is concerned. Results also showed our approaches are extremely efficient when calculating multiple models as opposed to the traditional methods. Basically, the training time saved compared to the traditional methods is proportional to the number of models need to be investigated.

We believe this to be promising for big data for two main reasons: firstly, the coefficients of the models could be easily obtained as long as the SS arrays are calculated. Secondly, most of the models require a large amount of training and retraining, tuning and re-tuning to get better performance. While, multiple learning is able to solve or largely alleviate this time consuming problem.

Multiple learning approaches can be implemented on a single node as well as parallel computing frameworks, e.g. Spark. Due to time and resource constraints, our work is currently limited to closed-form solutions. For our further work, we would like to conduct more experiments over large scale datasets form real world applications and extend the multiple learning to models with no closed-form solutions.







{\footnotesize
\bibliographystyle{IEEEtran}
\bibliography{cites}
}

\end{document}